\documentclass{article}
\usepackage{ijcai23}

\usepackage{times}
\usepackage{soul}
\usepackage{url}
\usepackage[hidelinks]{hyperref}
\usepackage[utf8]{inputenc}
\usepackage[small]{caption}
\usepackage{graphicx}
\usepackage{amsmath}
\usepackage{amsthm}
\usepackage{booktabs}
\usepackage[switch]{lineno}

\usepackage[linesnumbered,ruled,vlined,noend]{algorithm2e}
\usepackage{enumitem}
\usepackage{subcaption}
\usepackage{bbm}
\usepackage{url}
\usepackage{multirow}

\newtheorem{theorem}{Theorem}

\theoremstyle{definition}
\newtheorem{definition}{Definition}[section]
\newtheorem{problem}{Problem}

\newcommand{\avgcostproblem}{\textsc{a-dcfp}}
\newcommand{\completiontimeproblem}{\textsc{ct-dcfp}}
\newcommand{\partialavgcostproblem}{\textsc{o-dcfp}}
\newcommand{\algo}{\textsf{MIP-LNS}}
\newcommand{\algosim}{\textsf{MIP-LNS-SIM}}

\title{Simulation-Assisted Optimization for Large-Scale Evacuation Planning with Congestion-Dependent Delays}

\author{
Kazi Ashik Islam\and
Da Qi Chen\and
Madhav Marathe\and
Henning Mortveit\and 
Samarth Swarup\And
Anil Vullikanti
\affiliations
Biocomplexity Institute, University of Virginia
\emails
\{ki5hd, wny7gj, marathe, henning.mortveit, swarup, vsakumar\}@virginia.edu
}

\begin{document}

\maketitle

\begin{abstract}
Evacuation planning is a crucial part of disaster management.
However, joint optimization of its two essential components, routing and scheduling, with objectives such as minimizing average evacuation time or evacuation completion time, is a computationally hard problem.
To approach it, we present \algo{}, a scalable optimization method that utilizes heuristic search with mathematical optimization and can optimize a variety of objective functions. We also present the method \algosim{}, where we combine agent-based simulation with \algo{} to estimate delays due to congestion, as well as, find optimized plans considering such delays. 
We use Harris County in Houston, Texas, as our study area. We show that, within a given time limit, \algo{} finds better solutions than existing methods in terms of 
three different metrics.
However, when congestion dependent delay is considered, \algosim{} outperforms \algo{} in multiple performance metrics.
In addition, \algosim{} has a significantly lower percent error in estimated evacuation completion time compared to \algo{}.

\end{abstract}

\section{Introduction}
Evacuation plans are essential to ensure the safety of people living in areas that are prone to disasters such as floods, hurricanes, tsunamis and wildfires. Large-scale evacuations have been carried out during the past hurricane seasons in Florida, Texas, Louisiana, and Mississippi. Examples of hurricanes when such evacuations were carried out include, Katrina \& Rita (2005), Ike \& Gustav (2008), Irma \& Harvey (2017), Laura (2020), Ida (2021), and Ian (2022). 
For instance, about 2.5 million individuals were evacuated from the coastal areas of Texas~\cite{carpender2006urban} before the landfall of Hurricane Rita. The most recent category four hurricane, Ian, caused $119$ deaths in the state of Florida alone~\cite{hurricane_ian:2022}.
To ensure people can evacuate in a safe and orderly manner, a good evacuation plan needs to have two essential components: ($i$) Evacuation Routes, i.e. what roads to take, and ($ii$) Evacuation Schedule i.e. when to leave.

The focus of our paper is on {\em jointly optimizing the routes and schedules}. Informally, the idea is to find a schedule of when individuals can begin evacuation (within a given time window) and a route that would be used to evacuate, so as to minimize the objective functions capturing the system level evacuation time (see Section~\ref{sec:prelim} for formal definition of the problems). 
Jointly optimizing over the routes and schedule is significantly harder from a computational standpoint (See Section \ref{sec:hardness} for hardness results). Existing methods, even those designed to find bounded sub-optimal solutions, do not scale to city or county level planning problems. Thus, finding good evacuation routes and schedule within a reasonable amount of time, for a city or county with a large population, remains an open problem. 

Moreover, during evacuations, large number of people try to egress out of an area in a relatively small amount of time. This results in traffic congestion and huge delays in the evacuation process. It is crucial to consider and model such delays during the planning phase.
However, most of the existing works on finding optimized evacuation plans do not consider the slowdown of traffic caused by high traffic density. For instance, \cite{even2015convergent,romanski2016benders,hafiz2021large} all consider a constant travel time on each road, no matter how high (or low) the traffic density is on those roads. To overcome this, we treat the travel time on each road, as a parameter. We then utilize agent-based simulation, which is capable of modeling the complex relationship among the traffic density and speed on different roads, to learn the parameter values.

\subsection*{Our Contributions}
\textbf{First}, we present \algo{}, a scalable optimization method that can find solutions to a class of evacuation planning problems, while optimizing for a variety of objectives (Section \ref{sec:optimization}). It is designed based on the Large Neighborhood Search (LNS) framework. In this paper, we work with \emph{three objectives}: minimizing ($i$) average evacuation time, ($ii$) evacuation completion time, and ($iii$) average evacuation time of `non-outlier' evacuees. We show that all of these three optimization problems are hard to approximate within a logarithmic factor. In \algo{}, we model the problems as Mixed Integer Programs (MIP) and then find solutions to these programs using a combination of heuristic search and mathematical optimization. A key technical challenge involves modeling \textit{flows over time}; this is best achieved using time expanded graphs. But it also leads to substantial increase in the size of the MIP and results in a significant increase in computing resources (time and space). This necessitates the need to combine heuristic search methods developed in the AI literature with MIP techniques developed in the OR literature.

\textbf{Second}, we illustrate how our approach can scale to large problem sizes and can be applied to realistic real-world problems.  We choose Harris county in Houston, Texas as our study area and apply \algo{}. The county has about $1.5$ million households, spans an area of $1,778$ square miles, and has been affected by several hurricanes in the past. We use real-world road network data and a synthetic population data to construct a realistic problem instance. \emph{It is ten times larger (in terms of the size of the time expanded graph, and the number of evacuating vehicles) than the problem instance of our baseline~\cite{hafiz2021large}.} 
We show that, within a given time limit, \algo{} finds solutions for our problem instance that are on average $13\%$, $20.7\%$ and $58.43\%$ better than the baseline method 
in terms of average evacuation time, evacuation completion time and optimality guarantee of the solution, respectively (Section \ref{sec:algo_exec}).

\textbf{Finally}, we present \algosim{}, where we treat the travel time on each edge as a delay parameter and utilize agent-based simulation to learn the parameter values (Section \ref{sec:mip_lns_sim}). The MIP model, with the learned parameter values, is then solved to find optimized evacuation plans. Agent-based simulations provide a natural approach to capture the delays one incurs due to congestion -- \emph{dynamic flow problems cannot capture these delays.} 
Our approach is an example of methods that combine simulation and optimization methods (SO)~\cite{gosavi2015simulation,amaran2016simulation} considered in OR and has become increasingly popular in AI~\cite{van2013computational,kambhampati2020challenges,doppa2021adaptive}.
Through our experiments, we show that \algosim{} outperforms \algo{} in terms of average evacuation time, evacuation completion time, and average time spent on the road ($10\%, 17\%, 77\%$ improvement respectively) when delay due to congestion is considered (Section~\ref{sec:simulation}).
In addition, \algosim{} has a significantly lower percent error ($6\%$) in estimated evacuation completion time compared to \algo{} ($76\%$), demonstrating the efficacy of \algosim{} in evacuation planning subject to congestion constraints.

\section{Related Work}
\label{sec:lit_review}
Researchers have approached the evacuation planning problem in different ways in the past.~\cite{hamacher2002} formulated it as a dynamic network flow optimization problem and introduced the idea of time expanded graphs to solve it using mathematical optimization methods. However, their method had prohibitively high computational cost, which paved the way to several heuristic methods~\cite{Lu:2005:CCR:2156226.2156249,Kim:2007:ERP:1341012.1341039,Shahabi2014}. These methods solve the routing problem only -- they either do not consider the scheduling problem at all or propose simple schemes such as letting evacuees leave at a constant rate. On the other hand,  \cite{even2015convergent,romanski2016benders,hafiz2021large} have considered the joint optimization problem of routing and scheduling. They formulated the problem as Mixed Integer Programs and used decomposition techniques~\cite{benders1962partitioning,Magnanti1981} to separate the route selection and scheduling process. However, none of these works consider the slowdown of traffic at high traffic densities. A review of existing works on evacuation planning can be found in the survey paper~\cite{Bayram2016}.

The use of convergent evacuation routes has been explored in the literature~\cite{even2015convergent,romanski2016benders,hafiz2021large}, where all evacuees coming to an intersection follow the same path afterwards. This is also known as confluent flow~\cite{Chen2006}. \cite{Golin2017NonapproximabilityAP} investigated the single-sink confluent quickest flow problem where the goal is minimizing the time required to send supplies from sources to a single sink. They showed that the problem is hard to approximate within a logarithm factor. We prove that all the planning problems considered in this paper are also hard to approximate.

We use the most recent method (Benders Convergent or BC) by \cite{hafiz2021large} as our baseline and show that \algo{} finds better solutions in terms of three different metrics. In addition, we provide direct MIP formulations for three different objectives, all of which can be optimized using \algo{} (as well as \algosim{}) without needing any modifications. 

Heuristic search methods are generally applied to problems that are computationally intractable. The goal is to find good solutions in a reasonable amount of time. The Large Neighborhood Search (LNS) framework \cite{Shaw1998} has been successfully applied to various hard combinatorial optimization problems in the literature \cite{Pisinger2018}. Recently, \cite{Li2021} applied the LNS framework to find solutions for the Multi-Agent Path Finding Problem where the goal is to find collision free paths for multiple agents. 

Simulation models have been used in the literature
for finding optimal 
decision variables for a given objective function~\cite{dangelmaier2006simulation,Sajedinejad2011,Osorio2013,teufl2018optimised}. This is useful especially when the objective function's closed form is unknown or is too complex, but the function's value can be evaluated through (possibly time-expensive) simulation. In such cases, simulation models have been utilized as fitness functions within heuristic search and meta-heuristic algorithms~\cite{Sajedinejad2011,teufl2018optimised}. Existing research works have also proposed constructing a representative function, often termed as a \textit{metamodel}, of the actual objective function by using its calculated values from the simulation model \cite{Osorio2013,SoaresdoAmaral2022}. Parameters of the metamodel are learned by methods such as regression, 
artificial neural networks~\cite{gosavi2015simulation}. Metamodels are mainly constructed because it is possible or easier to optimize these models. To fulfill our goal to find optimized evacuation plans considering congestion-dependent delays, we use our MIP model as a metamodel and use agent-based simulation to learn its parameter values. Both parameter learning and metamodel optimization are performed within \algosim{}.

\section{Problem Formulation}
\label{sec:prelim}
In this section, we introduce some preliminary terms that we use in our problem formulation. 

\begin{definition}
A \textit{road network} is a directed graph $\mathcal{G} = (\mathcal{N}, \mathcal{A})$ where every edge $e \in \mathcal{A}$ has ($i$) a capacity $c_e$, representing the number of vehicles that can enter the edge at a given time and ($ii$) a travel time $T_e$ representing the time it takes to traverse the edge. 
\end{definition}

\begin{definition}
Given a road network, a \textit{single dynamic flow} is a flow $f$ along a single path with timestamps $a_v$, representing the arrival time of the flow at vertex $v$ that obeys the travel times. In other words, $a_v-a_u \geq T_{uv}$ for edge $(u,v)$. A \textit{valid dynamic flow} is a collection of single dynamic flows where no edge at any point in time exceeds its edge capacity. 
\end{definition}

\begin{definition}
An \textit{evacuation network} is a road network that specifies $\mathcal{E}, \mathcal{S}, \mathcal{T} \subset \mathcal{N}$, representing a set of source, safe and transit nodes respectively. Furthermore, for each source node $k\in \mathcal{E}$, let $W(k)$ and $d_k$ represent the set of evacuees and the number of evacuees at source $k$ respectively. Let $W$ denote the set of all evacuees.
\end{definition}

For scheduling an evacuation, we observe that once an evacuee has left their home, it is difficult for them to pause until they reach their destination. We also assume that people from the same location evacuate to the same destination.
Similarly, we assume that if two evacuation routes meet, they should both be directed to continue to the same location.

\begin{definition}
Given an evacuation network, we say a valid dynamic flow is an \textit{evacuation schedule} if the following are satisfied:
\begin{itemize}[leftmargin=*]
    \item all evacuees end up at some safe node, 
    \item no single dynamic flow has any intermediary wait-time (i.e. $a_v-a_u = T_{uv}$ and,
    \item the underlying flow (without considering time) is confluent, where if two single dynamic flows use the same vertex (possibly at different times), their underlying path afterwards is identical.
\end{itemize}
\end{definition}

Two natural objectives to minimize are the average evacuation time of the evacuees and the evacuation completion time.
To define these formally, let $t_i$ denote the evacuation time of evacuee $i$. We then formally define the following problems:

\begin{problem}{Average Dynamic Confluent Flow Problem (\avgcostproblem{})}.
Given an evacuation network, let $T_{max}$ represent an upper bound on evacuation time. Find an evacuation schedule such that all evacuees arrive at some safe node before time $T_{max}$ while minimizing $\frac{1}{|W|}\sum_{i \in W} t_i$. 
\end{problem}


\begin{problem}{Completion Time Dynamic Confluent FLow Problem (\completiontimeproblem{})}.
Given an evacuation network, find an evacuation schedule such that all evacuees arrive at some safe node while minimizing $\max_{i \in W} t_i$. 
\end{problem}


We define a third problem, Outlier Average Dynamic Confluent Flow Problem (\partialavgcostproblem{}), where the goal is to minimize average evacuation time of `non-outlier' evacuees. For brevity, its formal definition is provided in the supplementary materials~\cite{arXiv:2209.01535}.

\subsection{Time Expanded Graph for Capturing Flow Over Time}
Joint routing and scheduling over networks requires one to study \textit{flows over time}, as static flows make the unrealistic assumption that flows travel instantaneously (detailed discussion in the supplementary materials~\cite{arXiv:2209.01535}).
For this purpose, researchers have defined dynamic flows (\cite{skutella2009introduction,ford2015flows}) and used time expanded graphs to solve dynamic flow problems (\cite{romanski2016benders,hafiz2021large}). In this paper, we also use a time expanded graph (\textbf{TEG}) to capture the flow of evacuees over time.

\begin{figure}[!b]
    \centering
    \begin{subfigure}[b]{0.475\columnwidth}
        \includegraphics[width=\linewidth]{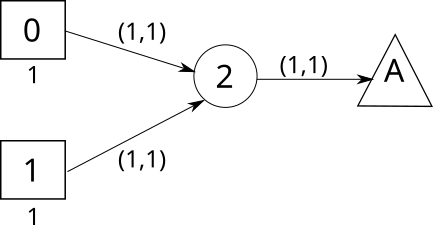}
        \caption{Sample Evacuation Network. Edges are labeled with travel time and flow capacity respectively. Source, safe and transit nodes are denoted by squares, triangles, and circles respectively. Source nodes are labeled with number of evacuees.}
        \label{fig:sample_network}
    \end{subfigure}
    \hfill
    \begin{subfigure}[b]{0.475\columnwidth}
        \includegraphics[width=\linewidth]{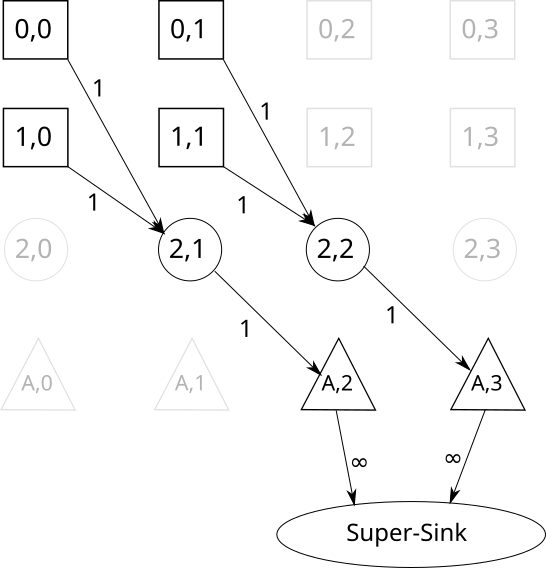}
        \caption{Time Expanded Graph (TEG) for the Sample Network. Edges are labeled with capacity. Construction of this TEG sets an upper bound of 3 time units for evacuation completion.}
        \label{fig:sample_teg}
    \end{subfigure}
    \caption{Sample Problem Instance}
    \vspace{-2mm}
\end{figure}

Time expanded graph is a directed graph denoted by $\mathcal{G}^x = (\mathcal{N}^x = \mathcal{E}^x \cup \mathcal{T}^x \cup \mathcal{S}^x, \mathcal{A}^x$). To construct it, we first fix a time horizon $\mathcal{H}$ and discretize the temporal domain into discrete timesteps of equal length. Then we create copies of each node at each timestep within~$\mathcal{H}$. After that, for each edge $e(u, v)$ in the road network, we create edges in the TEG as $e_t(u_t, v_{t+T_e})$ for each $t \leq \mathcal{H}-T_e$ where the edges $e_t$ have the same flow capacity as $e$. 
Finally, we add 
a super sink node~$v_t$ that connects to the nodes~$u_t$ for each~$u \in \mathcal{S}$ and each~$t \leq \mathcal{H}$. Edges to the super sink node are assigned an infinite amount of capacity.
Note that, when creating the time expanded graph, we are adding an additional dimension (i.e. time) to the road network. \emph{The size of the TEG is about $\mathcal{H}$ times as large as the road network in terms of number of nodes and edges.} -- yielding a
substantially larger problem representation. 

A sample evacuation network and its corresponding TEG with time horizon~$\mathcal{H} = 3$ are shown in Figure (\ref{fig:sample_network}-\ref{fig:sample_teg}). The source, safe and transit nodes are denoted by squares, triangles, and circles respectively. In the TEG, there may be some nodes that are ($i$) not reachable from the source nodes, or ($ii$) no safe node can be reached from these nodes within the time horizon. These nodes are greyed out in Figure \ref{fig:sample_teg}.
An optimal solution of \avgcostproblem{} (and \completiontimeproblem{}) for this sample problem instance is to use the routes $0 \rightarrow 2 \rightarrow A$ from source node~$0$ and $1 \rightarrow 2 \rightarrow A$ from source node~$1$, where the evacuee at source node $0$ and $1$ leave at timestep~$1$ and~$0$ respectively.


\subsection{Mixed Integer Program (MIP) Model}
\label{sec:mip_formulation}
Now, we present the Mixed Integer Program (\ref{obj:mip_objective}--\ref{constr:edge_assignment_binary}) that we use to represent a class of evacuation planning problems. 
We can have different objectives in the program (Objective \ref{obj:mip_objective}), each representing a certain planning problem. We use two types of variables: ($i$) Binary variable~$x_e, \forall e \in \mathcal{A}$, which will be equal to one if and only if the edge~$e$ is used for evacuation. Otherwise, it will be zero. ($ii$) Continuous variable~$\phi_{e_t}, \forall e_t \in \mathcal{A}^x$, which denotes the flow of evacuees on edge~$e$ at timestep~$t$. 

Constraint (\ref{constr:single_outgoing_edge}) ensures that there is exactly one outgoing edge from each evacuation node. Constraint (\ref{constr:at_most_one_outgoing_edge}) ensures that at each transit node, there is at most one outgoing edge. Constraint (\ref{constr:total_population}) enforces that the total flow coming out of every evacuation node is equal to the number of evacuees at the corresponding node. Constraint (\ref{constr:flow_conservation}) ensures flow conservation in the time-expanded graph; here, $\delta^{-}(i)$ and $\delta^{+}(i)$ denote the set of incoming and outgoing edges to/from node~$i$, respectively. Constraint (\ref{constr:flow_capacity}) allows flow on chosen edges only; it also enforces flow capacity on each edge of the time-expanded graph. Constraint (\ref{constr:flow_continuous}) defines $\phi$ as continuous and non-negative variable; constraint (\ref{constr:edge_assignment_binary}) defines $x$ as binary variable. 
The constraint that evacuation completion time needs to be less than the given upper bound is implicit in the model, as we set the time horizon of the TEG to this upper bound.

\begin{align}
    & \text{Objective to Optimize}\label{obj:mip_objective}\\
    & s.t. \sum_{e \in \delta^+(k)} x_e = 1 & \forall k \in \mathcal{E} \label{constr:single_outgoing_edge}\\
    &\sum_{e \in \delta^+(i)} x_e \leq 1 & \forall i \in \mathcal{T} \label{constr:at_most_one_outgoing_edge}\\
    &\sum_{e \in \delta^+(k)}\sum_{t \leq \mathcal{H}} \phi_{e_t} = d_k & \forall k \in \mathcal{E} \label{constr:total_population}\\
    &\sum_{e \in \delta^-(i)} \phi_{e} = \sum_{e \in \delta^+(i)} \phi_{e} & \forall i \in \mathcal{T}^x \cup \mathcal{S}^x \label{constr:flow_conservation}\\
    & \phi_{e_t} \leq x_e  c_{e_t} & \forall e \in \mathcal{A}, t \leq \mathcal{H} \label{constr:flow_capacity}\\
    &\phi_{e} \geq 0 & \forall e \in \mathcal{A}^x \label{constr:flow_continuous}\\
    &x_e \in \{0, 1\} & \forall e \in \mathcal{A} \label{constr:edge_assignment_binary}
\end{align}

\medskip

To solve \avgcostproblem{} using model (\ref{obj:mip_objective}--\ref{constr:edge_assignment_binary}), we represent the total evacuation time 
using the variables $x$ and $\phi$ as follows:
\begin{align}
    \text{Total Evacuation Time} = \sum_{e \in \delta^-{(v_t)}} \phi_e t_s(e) \label{eq:total_evac_time_obj}
\end{align}

Here, $t_s(e)$ denotes the timestep of the starting node of edge $e$. Note that, minimizing the average evacuation time and the total evacuation time are equivalent.
So, the \avgcostproblem{} objective would be: $\min_{x, \phi}\sum_{e \in \delta^-{(v_t)}} \phi_e t_s(e)$.

We have just provided details on how to formulate \avgcostproblem{} as a MIP. Details on \completiontimeproblem{} and \partialavgcostproblem{} are provided in the supplementary materials~\cite{arXiv:2209.01535}.

\section{Inapproximability Results}
\label{sec:hardness}

In this section, we show that the problems we consider are not only {\sf NP-hard} but also hard to approximate. 
A summary of the hardness results is found in Table \ref{tbl:hard}. 


\begin{table}[!h]
\begin{center}
\begin{tabular}{|p{2.25cm}|p{1.25cm}|p{2cm}|p{1.25cm}|}
    \hline
    Hardness & \multicolumn{3}{c|}{Problems} \\ 
    \cline{2-4}
    & \avgcostproblem & \completiontimeproblem & \partialavgcostproblem \\ 
    \hline
    {\sf $O(\log n)$-hard} to approx. & Thm. \ref{hardlog} & See \cite{Golin2017NonapproximabilityAP}  & Thm. \ref{hardlog} \\ 
    \hline
\end{tabular}
\caption{Summary of Hardness}
\label{tbl:hard}
\end{center}
\vspace{-5mm}
\end{table}







\begin{theorem}
\label{hardlog}
For \avgcostproblem{} and \partialavgcostproblem{} with many sources and one safe node, it is {\sf NP-hard} to approximate within a factor of $O(\log n)$. 
\end{theorem}

The proof of Theorem \ref{hardlog} is provided in the supplementary materials~\cite{arXiv:2209.01535}. \emph{In addition, we show that all the three problems remain {\sf NP-hard} even when we consider the road network $G$ to be
a sub-graph of a grid and all destinations are along the border.} Street networks in several city neighborhoods resemble such networks.
 
\section{Heuristic Optimization}
\label{sec:optimization}

\label{sec:optimizing_model}
In this section, we present the method \algo{} where we use MIP solvers in conjunction with heuristic search.

\begin{algorithm}[!b]
    \caption{\algo{} Method}
    \label{alg:iterative_mip_heuristic}
    \KwIn{Initial solution: \textit{sol}, Time Expanded Graph: $TEG$, Time horizon: $T$, Model to optimize: $model$, (\%) of routes to update: $p$, Number of Iterations: $n$, Positive number: $p_{inc}$}
    \KwOut{Solution of $model$}
    \For {1 to n \label{algline:iterations}}{
        Select $(100-p)\%$ of the source locations uniformly at random. Let their set be $S$. \label{algline:random_selection}\\
        Fix the routes from the source locations in $S$. Set $x_e=1$ if $e$ is on any of the routes from $S$ in \textit{sol}.\\
        \textit{sol} $\gets$ Solution of $model$ from a MIP solver \label{algline:gurobi_call}\\
        $T^{\prime}$ $\gets$ evacuation completion time for solution \textit{sol}\\
        \If{$T - T^{\prime} > +threshold$}{
            Update the $model$ by setting the time horizon to $T^{\prime}$.
            Prune $TEG$ and $model$ by removing: \label{algline:prune_start}\\ 
            ($i$) nodes that are unreachable from the evacuation nodes within time horizon $T^{\prime}$, and  \\
            ($ii$) nodes from which none of the safe nodes can be reached within time horizon $T^{\prime}$ \label{algline:prune_end}
        }
        $p \gets p + p_{inc}$\\
    }
    \Return \textit{sol}
\end{algorithm}

Within \algo{}, we first calculate an initial feasible solution in two steps: ($i$) calculating an initial convergent route set, and ($ii$) calculating the schedule that minimizes the target objective using the initial route set. For $(i)$, we take a shortest path from each source to its nearest safe node by road. This is done using Algorithm 3 (in the supplementary materials~\cite{arXiv:2209.01535}) to make sure the route set is convergent. For $(ii)$, we use the just calculated route set to fix the binary variables $x_e$ in model (\ref{obj:mip_objective}-\ref{constr:edge_assignment_binary}). This gives us a linear program that can be solved optimally to get the schedule. 

Next, we search for better solutions in the neighborhood of the solution at hand (Algorithm \ref{alg:iterative_mip_heuristic}). Here, we run $n$ iterations. In each iteration, we select $q=(100-p)\%$ of sources uniformly at random and keep their routes fixed. This reduces the size of the MIP as we have fixed values for a subset of the variables. We then optimize the `reduced' MIP model using a MIP solver~\cite{gurobi}. Essentially, we are searching for a better solution in the neighborhood where the selected $q\%$ routes are already decided. Any solution found in the process will also be a feasible solution for the original problem. If we find a better solution with an evacuation completion time $T^{\prime}$ that is less than the current time horizon ($T$), then we also update the model by setting the time horizon to $T^{\prime}$. When resetting the time horizon, we prune the TEG and the MIP model (lines \ref{algline:prune_start}--\ref{algline:prune_end}). This reduces the number of variables in the MIP model and simplifies the constraints. At the end of each iteration, we increase the value of $p$ by $p_{inc}$ amount. Note that, when $p=100$, we will be solving the original optimization problem. In our experiments, we start with $p=75$ and set $p_{inc} = 0.5$.

When solving the reduced problem in each iteration (line \ref{algline:gurobi_call}), we use ($i$) a time limit, and ($ii$) a parameter $threshold\_gap$ to decide when to stop. MIP solvers keep track of an upper bound ($Z_U$) (provided by the current best solution) and a lower bound ($Z_L$) (obtained by solving relaxed LP problems) of the objective value.
We stop the optimization when the relative gap $(Z_U - Z_L) / Z_U$ becomes smaller than the $threshold\_gap$. In our experiments, we set this to $5\%$.  In total, \algo{} has four parameters: $n$, $p$, $p_{inc}$, and $threshold\_gap$. 

\section{Simulation-Assisted Model for Optimization (\algosim{})}
\label{sec:mip_lns_sim}
In our formulation (Section \ref{sec:prelim}), we assume the travel time on each edge to be a constant. 
However, in practice, travel time on a road is affected by the number of vehicles on it (i.e. traffic density). Moreover, travel time on an edge also affects how many cars can enter it in a given amount of time (i.e. the flow capacity) \cite[Chapter\ 5]{mannering2020principles}. 
We, therefore, treat the edge travel times as `parameters' and aim to learn suitable values of these parameters to realistically model congestion-dependent delays.

To estimate the parameters, we use the agent-based queuing network simulation system QueST~\cite{Islam2020} with the logistic traffic model. The simulator is able to model the complex relationship between traffic density and effective speed of vehicles on the road. 
Given the routes and schedule, we simulate the evacuation process using QueST and determine the average travel time on each edge used during evacuation. This provides us a reasonable estimate of travel time on the edges when certain routes and schedule are used. However, simulating the evacuation of the entire population is a time consuming task. Therefore, we only simulate the evacuation of a certain percentage $(p_e)$ of the evacuees at each source. To be
more precise, we simulate the departure of the first $p_e\%$ of the evacuees from each source, following the evacuation schedule. 
Our intuition is: congestion faced by the first $p_e\%$ of evacuees provides us a good estimation of the overall congestion faced by all evacuees throughout the entire evacuation. This is because
people who leave first should not overlap too much with people who leave much later.


\begin{algorithm}[t]
    \caption{\algosim{} Method}
    \label{alg:mip_lns_sim}
    \KwIn{Evacuation network: $\mathcal{G}$, Initial solution: \textit{sol}, Number of iterations: $m$, Percentage of Evacuees to simulate: $p_e$}
    \KwOut{Evacuation routes and schedule.}
    \For{each edge $e \in \mathcal{A}$}{ \label{algline:max_speed_start}
        $T_e \gets$ Time it takes to traverse $e$ at speed limit.\\
        $c_e \gets$ Updated flow capacity of $e$.\label{algline:max_speed_end}
    }
    \For {$1$ to $m$}{
        $TEG \gets$ Time expanded graph of $\mathcal{G}$ with current travel time and capacity values of the edges.\\
        $model \gets$ MIP model (\ref{obj:mip_objective}--\ref{constr:edge_assignment_binary}) from $\mathcal{G}$ and $TEG$.\\
        $sol \gets$ Solution of $model$ from \algo{}. \label{algline:sol}\\
        Simulate evacuation of first $p_e \%$ evacuees at each source with routes and schedule from $sol$.\\
        \For{each edge $e \in \mathcal{A}$ used in $sol$}{
            $T_e \gets$ Avg. travel time on $e$ from simulation. \label{algline:parameter_estimation}\\
            $c_e \gets$ Updated flow capacity of $e$.
        }
    }
    \Return \textit{sol}
\end{algorithm}

Based on the above idea, we present the method \algosim{} (Algorithm \ref{alg:mip_lns_sim}). Initially, we assume that vehicles travel on each edge at the maximum speed allowed and calculate the travel time (i.e. the parameters) and flow capacity accordingly (line \ref{algline:max_speed_start}--\ref{algline:max_speed_end}). We then create the time-expanded graph based on these values and construct the MIP (i.e. our \textit{metamodel}). Next, We solve the MIP using \algo{}. We use the routes and schedule given by the solution to simulate the evacuation of first $p_e\%$ of the evacuees at each source. From the simulation results, we calculate the average travel time on each edge used in the solution and update the travel time as well as the flow capacity of these edges (details in supplementary materials~\cite{arXiv:2209.01535}). We do this iteratively for $m$ times. In our experiments, we have used $p_e=5, 10$ and $m=10$. Note that, both parameter estimation (line \ref{algline:parameter_estimation}) and the metamodel optimization (line \ref{algline:sol}) are performed within \algosim{}.

\section{Experiments}
\label{sec:experiments}
In this section, we present details of our problem instance and our experiment results.

\subsection{Problem Instance}
\label{sec:problem_instance}
We use Harris County in Houston, Texas as our study area. 
We have used data from HERE maps~\cite{HERE} to construct its road network. The network contains roads of five (1 to 5) different function classes, which correspond to different types of roads. For instance, function class 1 roads are major highways, and function class 5 roads are residential roads. 

For our experiments, we consider the nodes which connect and lead from function class 3/4 roads to function class 1/2 roads as the start/source locations of the evacuees. We then consider the problem of ($i$) when should evacuees target to enter the function class 1/2 roads and ($ii$) how to route them through the function class 1/2 roads to safely. As safe locations, we selected six locations at the periphery of Harris County which are on major roads. A visualization of the dataset is presented in Figure \ref{fig:harris_county}. Additional details regarding the problem instance are provided in Table \ref{tab:harris_county}.

\begin{figure}[!t]
    \centering
    \includegraphics[width=0.45\textwidth]{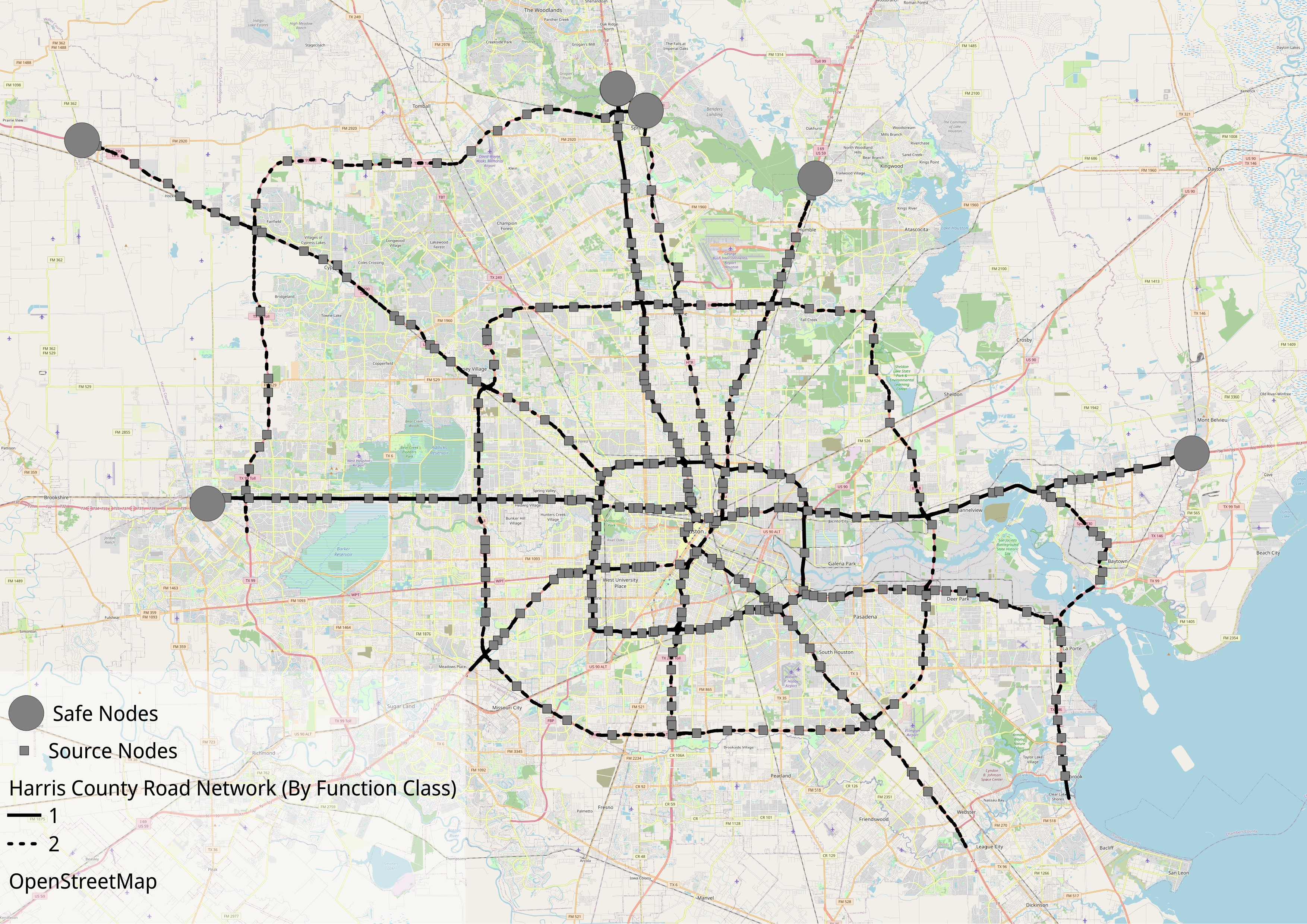}
    \caption{Harris County Problem Instance}
    \label{fig:harris_county}
    \vspace{-1mm}
\end{figure}

\begin{table}[!h]
    \small
    \centering
    \begin{tabular}{p{5.25cm}p{2.25cm}}
    \toprule
        \# of nodes, edges in the road network & $1338, 1751$\\
        \midrule
        \# of (evacuee) source locations & $374$\\
        \midrule
        \# of Households in the study area & $\sim1.5$M\\
        \midrule
        \midrule
        Time Horizon & $15$ Hours\\
        \midrule
        Length of one time unit & $2$ minutes\\
        \midrule
        \# of nodes, edges in the TEG & $684.7$K, $841.6$K\\
        \midrule
        \midrule
        \# of binary, continuous variables in \newline \avgcostproblem{} MIP & $1751$, $843.7$K\\
        \midrule
        \# of Constraints in \avgcostproblem{} MIP & $1.4$M\\
    \bottomrule
    \end{tabular}
    \caption{Problem Instance Details}
    \label{tab:harris_county}
    \vspace{-1mm}
\end{table}


We use a synthetic population \cite{adiga15US} to find the location of the households. We then assign the nearest exit ramp to each household as their source location. We assume that one vehicle is used per household for evacuation.

\subsection{\algo{} Results and Baseline Comparisons}
\label{sec:algo_exec}

We performed all our experiments and subsequent analyses on a high-performance computing cluster, with 128GB RAM and 4 CPU cores allocated to our tasks. In addition to \algo{}, we used two more methods to solve \avgcostproblem{}. We used a time limit of one hour for each method and compared the best solutions found within this time. The three methods we experimented with are described here.

\begin{table}[!h]
    \centering
    \begin{tabular}{p{1.5cm}p{1.35cm}p{1.35cm}p{2.75cm}}
    \toprule
    Baseline & \multicolumn{2}{c}{\algo{}} & Improvement\\
    \cmidrule{2-3}
     & Avg. & Std. Dev. & Over Baseline (\%)\\
    \midrule
    \multicolumn{4}{l}{\textbf{Average evacuation time (hours)}}\\
    \midrule
    2.54 & 2.21 & 0.06 & 13\\
    \midrule
    \multicolumn{4}{l}{\textbf{Evacuation completion time (hours)}}\\
    \midrule
    7.83 & 6.21 & 0.35 & 20.69\\
    \midrule
    \multicolumn{4}{l}{\textbf{Optimality guarantee (\%)}}\\
    \midrule
    20.47 & 8.51 & 2.43 & 58.43\\
    \bottomrule
    \end{tabular}
    \caption{\algo{} results for \avgcostproblem{} over ten experiment runs and comparison with the baseline method 
    in terms of three metrics: average evacuation time, evacuation completion time and optimality guarantee. On average, we see a $\sim13\%$, $\sim21\%$, and $\sim 58\%$ improvement in the three metrics respectively.}
    \label{tab:results_comparison}
\end{table}

\begin{enumerate}[leftmargin=*]
\item \textbf{Gurobi} In this experiment, we used Gurobi to directly solve model (1-8) with the \avgcostproblem{} objective. Gurobi was not able to find any feasible solution within the time limit. However, Gurobi was able to come up with a lower bound for the objective value. We used the lower bound to calculate the optimality guarantee of the solutions.

\item \textbf{Benders Decomposition} \cite{hafiz2021large} presented Benders Convergent (BC) method to solve the `Convergent Evacuation Planning' problem. Their problem is similar to \avgcostproblem{}, differing in the objective function, which is maximizing flow of evacuees instead of minimizing average evacuation time. We repurposed their method and used it as our baseline. 


\item \textbf{\algo{}} In our experiments with \algo{}, for \avgcostproblem{}, we used thirty iterations (i.e. $n=30$ in Algorithm \ref{alg:iterative_mip_heuristic} line \ref{algline:iterations}). Also, since we have a random selection process within \algo{}, we ran ten experiment runs with different seeds. 
\end{enumerate}

To compare the quality of our solutions with the baseline, we use three metrics: average evacuation time, evacuation completion time, and optimality guarantee. \textit{Optimality guarantee} is defined as follows: let the objective value of the solution \textit{sol} be $z_{sol}$ and the optimal objective be $z_{opt}$. Then, the optimality guarantee of \textit{sol} is $(z_{sol} - z_{opt})/z_{sol}$, i.e. the smaller the value of optimality guarantee, the better. If $z_{opt}$ is unknown, we can use a lower bound of it. Table \ref{tab:results_comparison} shows a comparison of our solutions with the baseline in terms of the three metrics.  

Let the value of a metric $m$ for the baseline and the \algo{} solution be $m_{base}$ and $m_{lns}$ respectively. Then, we quantify the improvement over the baseline as $(m_{base} - m_{lns})/m_{base}$. On average, we see an improvement of $13\%$, $21\%$, and $58\%$ over the baseline in the three above-mentioned metrics respectively. This indicates that \algo{} finds better solutions than the baseline within the given time limit. 

We also applied \algo{} to find solutions of \completiontimeproblem{} and \partialavgcostproblem{} for our problem instance. Due to limited space, we provide the results in the supplementary materials~\cite{arXiv:2209.01535}. In general, the experiment results show that \algo{} can effectively solve the problems with different objectives.




\subsection{\algosim{} Results}
\label{sec:simulation}

Within \algosim{}, we set the parameter $m=10$. We then experimented with two values for the parameter $p_e$, which are $5, 10$. We refer to these two settings as \algosim{}-$5\%$ and \algosim{}-$10\%$. We ran \algosim{} with the two settings and found two different solutions. We then performed agent-based simulations of the entire evacuation process (i.e. evacuate $100\%$ of the evacuees) using a solution of \algo{} and then also using solutions from \algosim{}-$5\%$ and \algosim{}-$10\%$. We used the QueST simulator with the logistic traffic model here. We now compare the simulation results.

\begin{figure}[!b]
    \centering
    \begin{subfigure}[b]{0.49\columnwidth}
        \includegraphics[width=\linewidth]{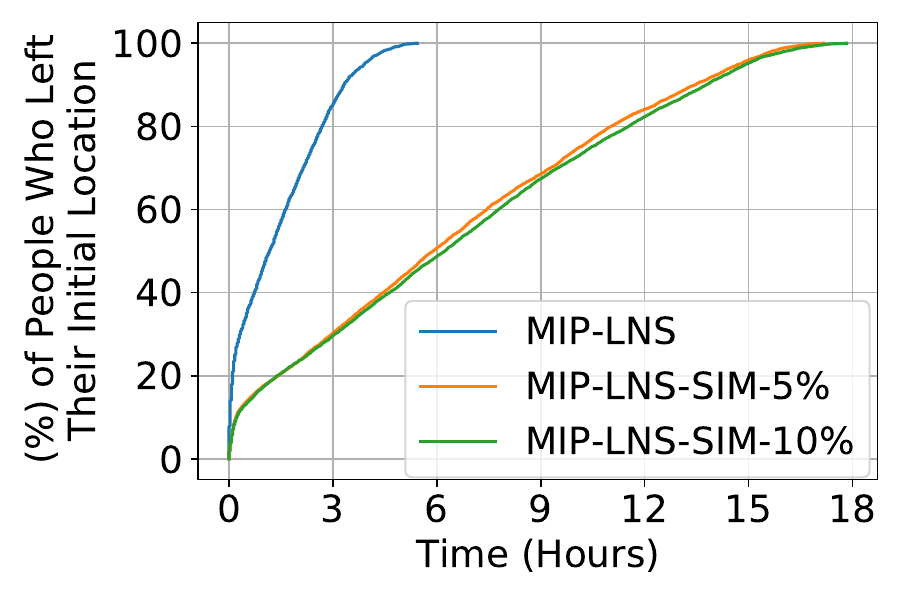}
        \caption{Departure rate of evacuees. We observe that in the \algo{} solution, evacuees leave very early compared to the solutions from \algosim{}-$5\%$ and \algosim{}-$10\%$. 
        }
        \label{fig:departure_rate_comp}
    \end{subfigure}
    \hfill
    \begin{subfigure}[b]{0.49\columnwidth}
        \includegraphics[width=\linewidth]{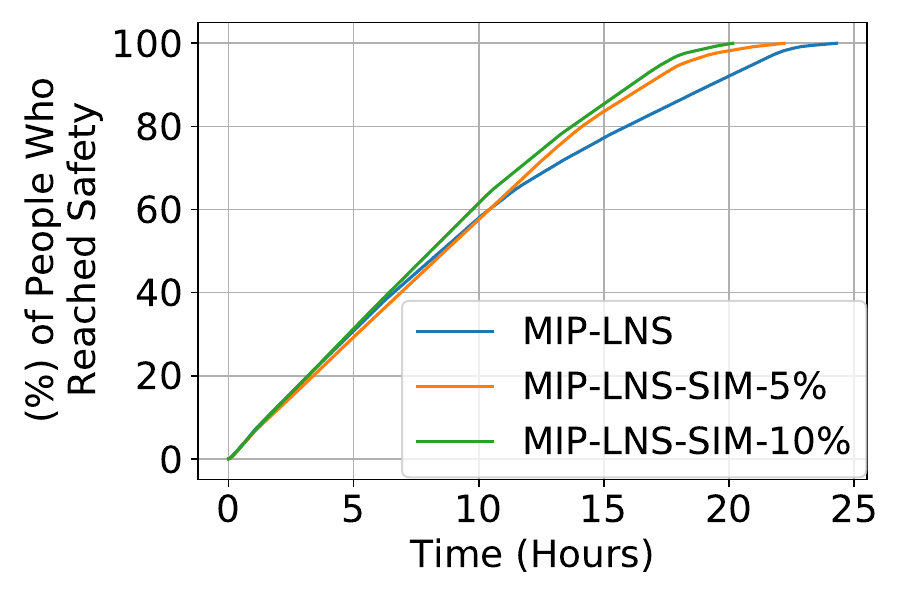}
        \caption{Arrival rate of evacuees at safe locations. \algosim{}-$10\%$ has the best evacuation rate and evacuation completion time, followed by \algosim{}-$5\%$ and then \algo{}.
        }
        \label{fig:evacuation_rate_comp}
    \end{subfigure}
    \caption{Comparison of \algo{}, \algosim{}-$5\%$, and \algosim{}-$10\%$ in terms of departure rate from sources and arrival rate at safe nodes. Even though \algosim{}-$5\%$, and \algosim{}-$10\%$ regulates the departure of evacuees, they evacuate everyone faster than \algo{}.}
\end{figure}

\begin{figure}[!t]
    \centering
    \begin{subfigure}[b]{0.49\columnwidth}
        \includegraphics[width=\linewidth]{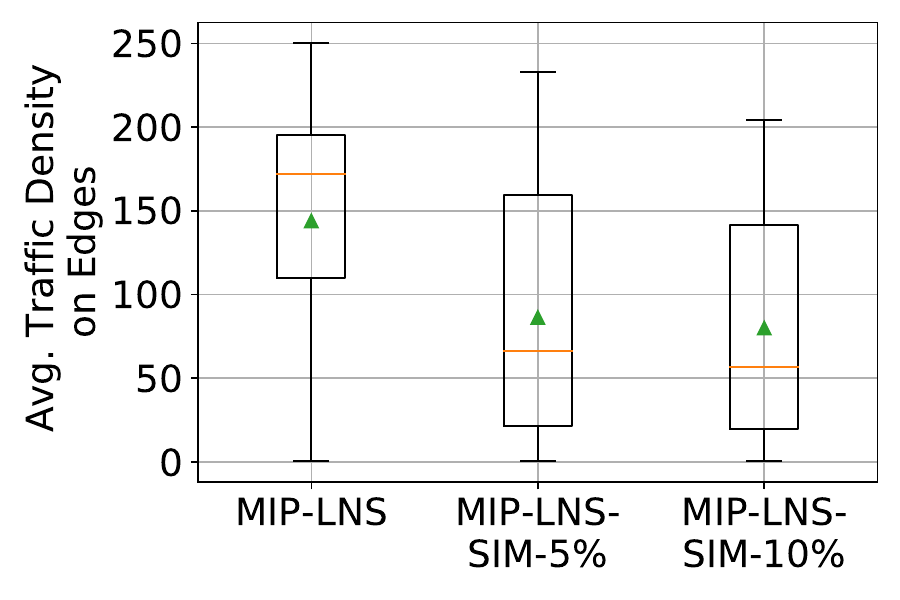}
        \caption{Boxplots showing average traffic density (NO. of vehicles per lane per km) on edges over the evacuation time period. Average traffic density is considerably low in the \algosim{}-$10\%$ and \algosim{}-$5\%$ solutions compared to \algo{}.}
        \label{fig:traffic_density_comp}
    \end{subfigure}
    \hfill
    \begin{subfigure}[b]{0.49\columnwidth}
        \includegraphics[width=\linewidth]{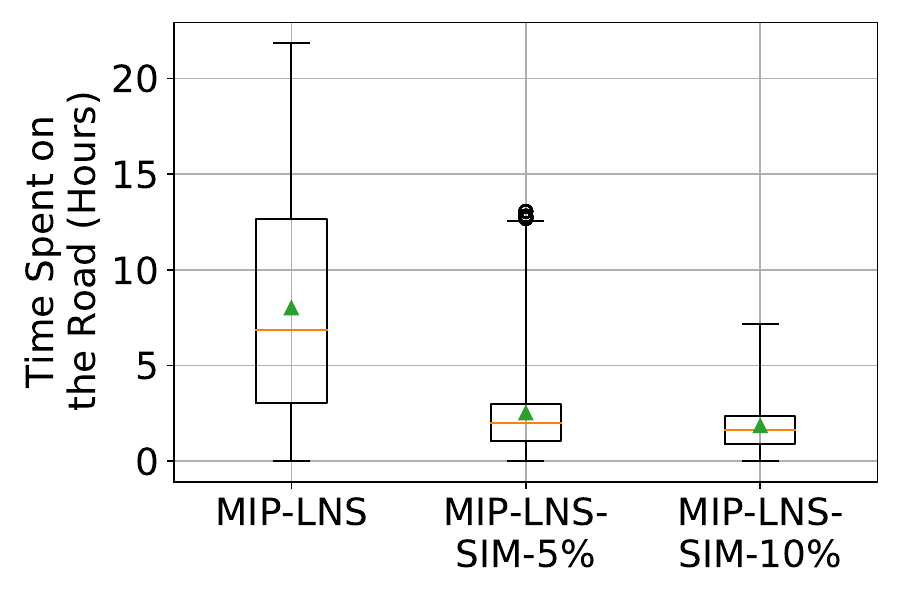}
        \caption{Boxplots showing time spent on the road by evacuees. Due to congestion (as in seen Figure \ref{fig:traffic_density_comp}), evacuees spend a significantly larger amount of time on the road in the \algo{} solution, compared to \algosim{}-$10\%$, \algosim{}-$5\%$ solutions.}
        \label{fig:transit_time_comp}
    \end{subfigure}
    \caption{Congestion on the roads in terms of traffic density and time spent on the road by evacuees.}
    \label{fig:traffic_density_and_transit_time_comp}
\end{figure}


Figure (\ref{fig:departure_rate_comp}) shows the departure rate of the evacuees from their initial locations, in the final solution of the three settings. We observe that, \algosim{}-$5\%$, and \algosim{}-$10\%$ regulates the departure of evacuees to a significant extent (compared to \algo{}). As evacuees leave late in these solutions, we might expect the evacuation completion time to be higher in these solutions compared to \algo{}. Surprisingly, we observe in Figure (\ref{fig:evacuation_rate_comp}) that the evacuation completion time is actually smaller in \algosim{}-$5\%$, and \algosim{}-$10\%$ compared to \algo{}.
This implies that although evacuees left early in the \algo{} solution, they could not reach safety early due to the resulting congestion on the roads. In the \algosim{}-$5\%$ and \algosim{}-$10\%$ solutions, evacuees departed from their initial location over a longer period of time. This way there was less congestion on the road and the evacuation was completed early even though many people started late. 

Figure \ref{fig:traffic_density_and_transit_time_comp} verifies our last statement. We see that traffic density on the edges (i.e. number vehicles per lane and per km) is higher in the \algo{} solution (than \algosim{}-$5\%$, \algosim{}-$10\%$) throughout the evacuation time period. The higher traffic density then causes the evacuees to spend more time on the road. 
In summary, \algosim{} is better than \algo{} in terms of average evacuation time, evacuation completion time, and average time spent on the road ($10\%, 17\%, 77\%$ improvement respectively, detailed results in the supplementary materials~\cite{arXiv:2209.01535}). These results indicate that \algosim{} is better at evacuation planning than \algo{} in terms of reducing congestion on the roads. 

\begin{table}[!h]
    \centering
    \begin{tabular}{p{2.95cm}p{1.45cm}p{1.45cm}p{1cm}}
    \toprule
    Algorithm & Estimated \newline ECT (hrs) & Simulated \newline ECT (hrs) & Percent \newline Error\\
    \midrule
    Baseline & 7.83 & 30.28 & 74.14\\
    \midrule
    \algo{} & 5.77 & 24.29 & 76.25\\
    \midrule
    \algosim{}-$5\%$ & 18.43 & 22.2 & 16.98\\
    \midrule
    \algosim{}-$10\%$ & 18.97 & 20.15 & 5.86\\
    \bottomrule
    \end{tabular}
    \caption{Estimated and simulated Evacuation Completion Time (ECT) in hours for the three settings and the baseline. The percent error of the estimated ECT decreases significantly with $p_e=5, 10$. This shows the effectiveness of \algosim{} in capturing delays due to congestion.}
    \label{tab:estimation_error}
\end{table}

Finally, Table \ref{tab:estimation_error} shows the estimated and the simulated evacuation completion time for Baseline~\cite{hafiz2021large}, \algo{}, \algosim{}-$5\%$, and \algosim{}-$10\%$. For instance, \algo{} predicts that when using the routes and schedule provided by its solution, the evacuation will be completed within $5.77$ hours. However, when simulated, it actually took $24.29$ hours to evacuate everyone. We also observe that the percent error in estimation decreases considerably in \algosim{}-$5\%$ and \algosim{}-$10\%$ where we have used $p_e=5$ and $10$ respectively. The lower percent error is earned at a cost of higher algorithm run time (\algosim{}-$5\%$: $\sim 6.5$ hours, and  \algosim{}-$10\%$: $\sim 7.55$ hours).

\section{Conclusion}
\label{sec:conclusion}
In this paper, we have presented an optimization method \algo{} to solve a class of evacuation planning problems. We demonstrated its efficacy by applying it on our study area of Harris county, Houston, Texas. We showed that, for our problem instance, \algo{} finds better solutions than the baseline method in terms of three different metrics. We have also presented \algosim{} to capture congestion-dependent delays. Through our experiments we have showed that \algosim{} outperforms \algo{} in terms of multiple metrics when congestion-dependent delay is considered. Our method can be incorporated into disaster management systems for effective evacuation planning. Additionally, it can help assess social vulnerability\footnote{https://www.atsdr.cdc.gov/placeandhealth/svi/index.html} of regions. 

\section*{Acknowledgments}
This work was partially supported by University of Virginia Strategic Investment Fund SIF160, the NSF Grants: CCF-1918656, OAC-1916805, RISE-2053013, and the NASA Grant 80NSSC22K1048.

\bibliographystyle{named}
\bibliography{ijcai23}

\clearpage

\appendix
\section{Time Expanded Graph for Capturing Flow Over Time}
Joint routing and scheduling over networks requires one to study \textit{flows over time}, as static flows make the unrealistic assumption that flows travel instantaneously.

\begin{figure}[!h]
    \centering
    \begin{subfigure}[b]{0.475\columnwidth}
        \includegraphics[width=\linewidth]{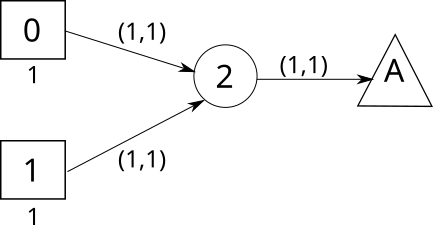}
        \caption{Sample Evacuation Network. Edges are labeled with travel time and flow capacity respectively. Source, safe and transit nodes are denoted by squares, triangles, and circles respectively. Source nodes are labeled with number of evacuees.}
        \label{fig:sample_network_appendix}
    \end{subfigure}
    \hfill
    \begin{subfigure}[b]{0.475\columnwidth}
        \includegraphics[width=\linewidth]{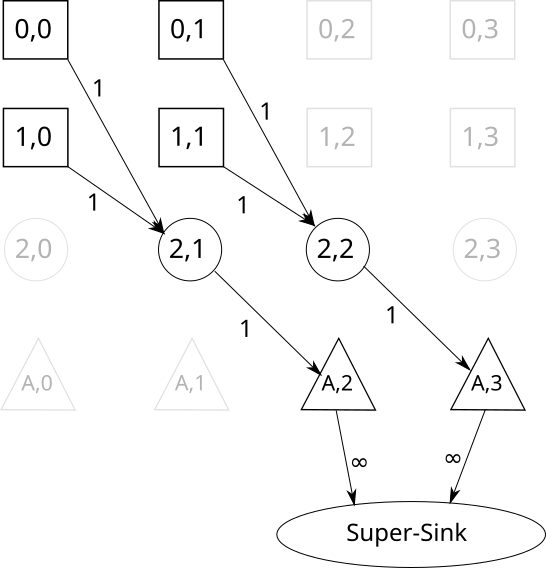}
        \caption{Time Expanded Graph (TEG) for the Sample Network. Edges are labeled with capacity. Construction of this TEG sets an upper bound of 3 time units for evacuation completion.}
        \label{fig:sample_teg_appendix}
    \end{subfigure}
    \caption{Sample Problem Instance}
    \vspace{-2mm}
\end{figure}

For instance, let us consider the sample evacuation network shown in Figure \ref{fig:sample_network_appendix}. All three edges in this network have a capacity of $1$, which means $1$ car can enter the link in a single timestep. However, sending flow from both sources ($0$ and $1$) at a rate of $1$ evacuee per timestep will not work because then two evacuees will reach node $2$ at the same timestep but only one evacuee will be able to enter edge $(2, A)$. The main issue here is that we need to keep track of available capacity on the edges at different points in time. We cannot do it using static flows because static flows have the underlying assumption that flows travel instantaneously.  
For this purpose, researchers have defined dynamic flows (\cite{skutella2009introduction,ford2015flows}) and used time expanded graphs to solve dynamic flow problems (\cite{romanski2016benders,hafiz2021large}). In this paper, we have also used a time expanded graph (\textbf{TEG}) to capture the flow of evacuees over time.

\section{Formulating \completiontimeproblem{} and \partialavgcostproblem{} as Mixed Integer Programs}
\label{sec:problem_variants}
In this section, we present the details of how we formulate \completiontimeproblem{} and \partialavgcostproblem{} as MIPs. Note that, both of these problems can be solved using \algo{} without making any changes to the algorithm.

\subsection{\completiontimeproblem{}}
Evacuation completion time cannot be represented as a linear function of the variables $x, \phi$. To optimize for this objective we modify the time expanded graph as follows: we add a new node $z_t$ for each timestep $t \leq \mathcal{H}$. Then, we add an edge from each safe node $(node, t) \in \mathcal{S}^x$ to the node $z_t$. Finally, we add add an edge from node $z_t, \forall t \leq \mathcal{H}$ to the super-sink $(v_t)$. Figure \ref{fig:ect_variant} shows an example with two safe nodes on how to do this construction.

\begin{figure}[!h]
    \centering
    \includegraphics[width=0.3\textwidth]{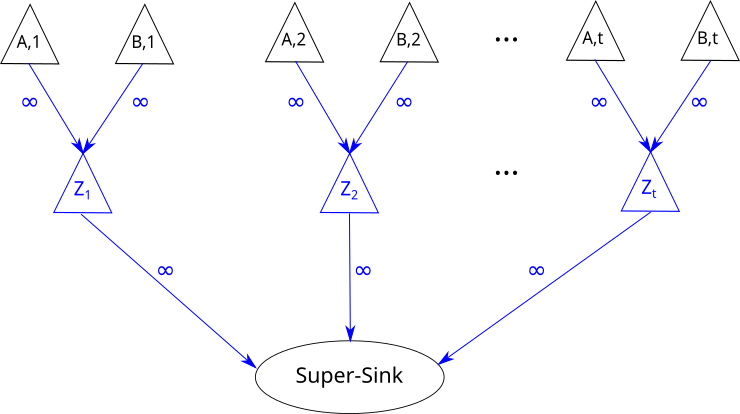}
    \caption{Modification of the time expanded graph for minimizing evacuation completion time. In this example, we have two safe nodes A and B. Blue nodes and edges are newly added to the TEG. Edges are labeled with their capacity.}
    \label{fig:ect_variant}
\end{figure}

With the new time expanded graph, we can represent the evacuation completion time (ECT) as follows:
\begin{align}
    \text{ECT} = \max_{e \in \delta^-(v_t)} \mathbbm{1}[\phi_e > 0] t_s(e) \label{eq:ect_with_indicator}
\end{align}

Here, $t_s(e)$ denotes the timestep of the starting node of edge $e$. $\mathbbm{1}$ denotes the indicator function, i.e. 
\begin{align}
    \mathbbm{1}[\phi_e > 0] = \begin{cases} \nonumber
      1  & \text{ if } \phi_e > 0 \\
      0 & \text{ otherwise}
    \end{cases}
\end{align}

As (\ref{eq:ect_with_indicator}) contains the indicator function, we introduce binary variables $y_e, \forall e \in \delta^-(v_t)$ and enforce that: ${\phi_e > 0 \iff y_e = 1}$. We can do it by adding the following constraints to model (\ref{obj:mip_objective}--\ref{constr:edge_assignment_binary}):
\begin{align}
    & \phi_e \geq \epsilon - M(1 - y_e) & \forall e \in \delta^-(v_t) \label{constr:big_M_1}\\
    & \phi_e \leq My_e & \forall e \in \delta^-(v_t) \label{constr:big_M_2}
\end{align}

Here, $\epsilon$ is a small positive constant (e.g. $0.001$) and $M$ is a positive constant which is equal to the total number of evacuees. With these new variables and constraints, we can now represent ECT as:
\begin{align}
    \text{ECT} = \max_{e \in \delta^-(v_t)} y_e t_s(e) \label{eq:ect_linear}
\end{align}

By minimizing (\ref{eq:ect_linear}) with model (\ref{constr:single_outgoing_edge}--\ref{constr:edge_assignment_binary},\ref{constr:big_M_1}--\ref{constr:big_M_2}) we can achieve the goal of minimizing evacuation completion time.

\subsection{\partialavgcostproblem{}} 
Minimizing the Average/Total Evacuation Time of $p$-fraction of the evacuees: In evacuation scenarios, some evacuees may be in such a situation that the cost of evacuating them may dramatically increase the overall evacuation objective, e.g. the average evacuation time. We consider such evacuees as `outliers'. A common way to handle outliers is to optimize the desired objective for the `non-outlier' evacuees, while taking into consideration that the `outlier' evacuees will also evacuate and use the same road network. We formally define the problem as follows:

\begin{problem}{Outlier Average Dynamic Confluent Flow Problem (\partialavgcostproblem{})}. 
Given an evacuation network, let $T_{max}$ represent an upper bound on evacuation time. Given $p \in [0, 1]$, find an evacuation schedule $S$ such that all evacuees arrive at some safe node before time $T_{max}$ while minimizing $\frac{1}{|W_p|}\sum_{i \in W_p} t_i$ where $W_p$ is the set of $p-$fraction of the evacuees with the lowest evacuation time in schedule $S$.
\end{problem}

To optimize the objective of \partialavgcostproblem{}, we extend the time expanded graph (TEG) as follows: First, we add a new `bypass' node $(A^{\prime}, 0)$ to the TEG. Then, from each safe node $s \in \mathcal{S}^x$, we add an edge to $(A^{\prime}, 0)$ with infinite capacity. Finally, we add an edge from $(A^{\prime}, 0)$ to the super-sink ($v_t$) with capacity $(\sum_{k \in \mathcal{E}} d_k) * (1 - p)$. Figure \ref{fig:outlier_variant} shows the newly constructed time expanded graph for our sample problem instance. With this new TEG, minimizing objective (\ref{eq:total_evac_time_obj}) will give us the desired solution. Here, the model solver will decide which $(1-p)$ fraction of the evacuees are the outliers and then direct them through the `bypass' node. As the timestep of this bypass node is zero, the cost for the outlier evacuees will not be added to the final objective. However, as they still have to reach the safe nodes through the road network, they will contribute to the congestion on the road.

\begin{figure}[!h]
    \centering
    \includegraphics[width=0.2\textwidth]{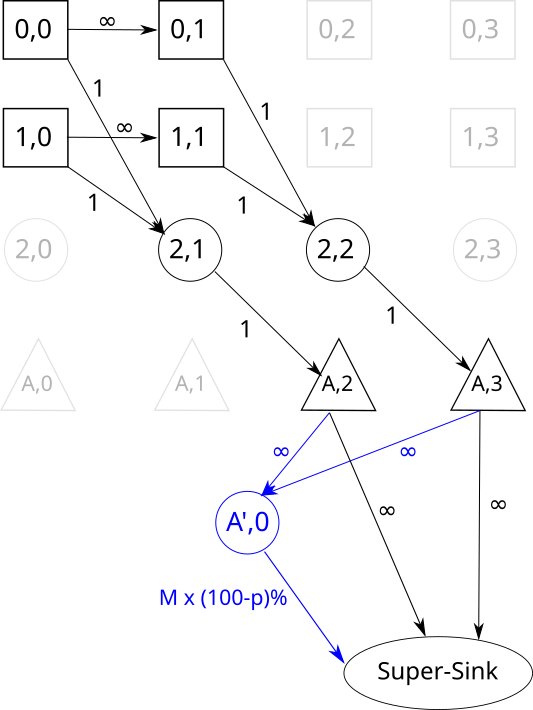}
    \caption{Extension of the time expanded graph for minimizing Average/Total Evacuation Time of $p\%$ of the evacuees. The blue node denotes the `bypass' node. Edges are labeled with their capacity. $M$ denotes the total number of evacuees.}
    \label{fig:outlier_variant}
\end{figure}


\section{Algorithms}
\label{sec:appendix_algorithms}
We use Algorithm \ref{alg:calc_convergent_routeset} to calculate the initial feasible solution for our heuristic search method \algo{}.
\begin{algorithm}[!h]
    \caption{Calculate Initial Convergent Route Set}
    \label{alg:calc_convergent_routeset}
    \KwIn{Road Network Graph: $\mathcal{G}(\mathcal{N}, \mathcal{A})$}
    \KwOut{A Convergent Route set}
    $\mathcal{A}^\prime \gets \{e(v, u) | e(u, v) \in \mathcal{A}\}$\\
    Construct the graph $\mathcal{G}^\prime(\mathcal{N}, \mathcal{A}^\prime)$.\\
    Add a node $v_{sink}$ to $\mathcal{G}^\prime$.\\
    Add edges $e(v_{sink}, v)$ to $\mathcal{G}^\prime, \forall v \in \mathcal{S}$ with $T_e = 0$.\\
    Calculate shortest paths in $\mathcal{G}^\prime$ from $v_{sink}$ to all nodes $v \in \mathcal{N}$ using a single source shortest path algorithm.\\
    \textit{routeset} $\gets$ Shortest paths from $v_{sink}$ to all nodes $v \in \mathcal{E}$.\\
    Reverse the direction of each path in \textit{routeset} and remove node $v_{sink}$.\\
    \Return \textit{routeset}
\end{algorithm}



\section{\algo{} Progress Over Iterations for \avgcostproblem{}}
To visualize the progress of \algo{} over the iterations for \avgcostproblem{}, we look at the experiment run (out of the ten experiment runs) that returned the best solution. Figure \ref{fig:avg_evac_time_over_iterations} shows the average evacuation time of the evacuees at different iterations of \algo{} . At iteration zero, we have our initial solution that has an average evacuation time of $3.36$ hours and an evacuation completion time of $13.5$ hours. After thirty iterations, \algo{} returned a solution with an average evacuation time of $2.12$ hours and an evacuation completion time of $5.77$ hours.
\begin{figure}[]
    \centering
    \includegraphics[width=0.35\textwidth]{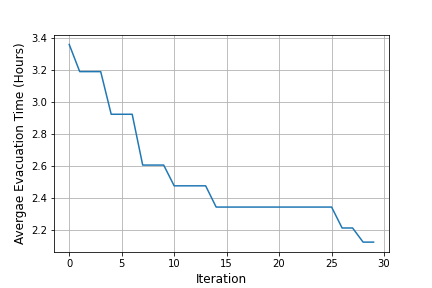}
    \caption{Average evacuation time over iterations. The decrease in average evacuation time indicates improvement of the objective over the iterations.}
    \label{fig:avg_evac_time_over_iterations}
\end{figure}

\section{\completiontimeproblem{} and \partialavgcostproblem{} Solutions}
\label{sec:variant_results}
In this section, we provide experiment results on the \completiontimeproblem{} and \partialavgcostproblem{} problems. 

We ran ten experiment runs (with different random seeds) of \algo{} for both \completiontimeproblem{} and \partialavgcostproblem{}. We again used the one hour time limit for these experiment runs. Figures \ref{fig:variants_aet}, \ref{fig:variants_ect}, \ref{fig:variants_non_aet} shows a comparison of the \avgcostproblem{}, \completiontimeproblem{}, and \partialavgcostproblem{} solutions in terms of the three metrics: average evacuation time, evacuation completion time, and non-outlier average evacuation time. Note that, these metrics are also the objectives of \avgcostproblem{}, \completiontimeproblem{}, and \partialavgcostproblem{} respectively. 

From the box-plots we observe that, in general, \avgcostproblem{}, \completiontimeproblem{}, and \partialavgcostproblem{} solutions are superior compared to the other solutions in terms of the objective they are designed to optimize, respectively. For instance, \avgcostproblem{} solutions are, in general, superior than \completiontimeproblem{} and \partialavgcostproblem{} solutions in terms of average evacuation time, but worse in terms of the other two objectives. This shows the effectiveness of \algo{} and of our formulation in solving evacuation planning problems, while optimizing for different objectives.

\begin{figure}[t]
    \centering
    \begin{subfigure}[b]{0.6\columnwidth}
        \includegraphics[width=\linewidth]{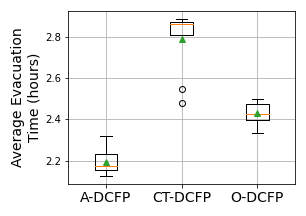}
        \caption{}
        \label{fig:variants_aet}
    \end{subfigure}
    \hfill
    \begin{subfigure}[b]{0.6\columnwidth}
        \includegraphics[width=\linewidth]{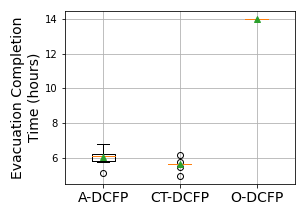}
        \caption{}
        \label{fig:variants_ect}
    \end{subfigure}
    \hfill
    \begin{subfigure}[b]{0.6\columnwidth}
        \includegraphics[width=\linewidth]{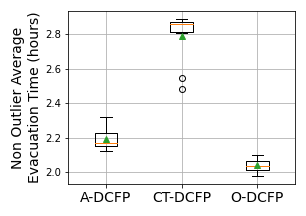}
        \caption{}
        \label{fig:variants_non_aet}
    \end{subfigure}
    \caption{Box-plots showing comparison of \algo{} solutions for \avgcostproblem{}, \completiontimeproblem{} and \partialavgcostproblem{} over ten experiment runs. Figures \ref{fig:variants_aet}, \ref{fig:variants_ect}, \ref{fig:variants_non_aet} show the comparison in terms of the three metrics average evacuation time, evacuation completion time, and non-outlier average evacuation time, respectively. We observe that the \avgcostproblem{}, \completiontimeproblem{}, and \partialavgcostproblem{} solutions are superior compared to the other solutions in terms of the objective they are designed to optimize, respectively.}
    \vspace{-3mm}
\end{figure}

\section{Calculating Flow Capacity of Edges}
Traffic flow, vehicle speed and traffic density are related by the following equation (Mannering and Washburn~\cite[Chapter 5, Equation 5.14]{mannering2020principles}):

\begin{align}
    q = uk \label{eq:flow_speed_density}
\end{align}

where:
\begin{align}
    q &= \text{Traffic flow (vehicles per time unit)} \nonumber\\
    u &= \text{vehicle speed (km per time unit)} \nonumber\\
    k &= \text{Traffic density (vehicles per km)} \nonumber
\end{align}

For an edge $e$, let:
\begin{align}
    l_e &= \text{length of edge $e$ (km)} \nonumber\\
    n_e &= \text{Number of lanes in edge $e$} \nonumber\\
    \hat{k}_e &= \text{Number of vehicles per lane per km on edge $e$} \nonumber
\end{align}

Then, we have:
\begin{align}
    k_e &= \hat{k}_e n_e
\end{align}

We assume that the maximum value of $\hat{k}$ for all edges is: 
\begin{align}
    & \hat{k}_{max} = 250 & \text{ [vehicles per lane per km]}\\
    \implies & k_{max} = 250n_e
\end{align}

Given a constant travel time $t_e$ for edge $e$, speed on the edge will be:
\begin{align}
    u_e = \frac{l_e}{t_e} \label{eq:speed_from_time}
\end{align}

Then, from Equation \ref{eq:flow_speed_density}:
\begin{align}
    q_{max} = u_e k_{max} \label{eq:flow_capacity}
\end{align}

Given a constant travel time for an edge, we use Equation \ref{eq:speed_from_time} and \ref{eq:flow_capacity} to calculate the flow capacity of the edge.

\section{Comparison of \algosim{} and \algo{} Considering Congestion Dependent Delays}
To compare the evacuation plans returned by \algosim{} and \algo{} with the consideration of congestion dependent delays, we run agent-based simulation of the entire evacuation (i.e. evacuate $100\%$ of the evacuees) using each of these plans. Here, we use the QueST simulator~\cite{Islam2020} with the logistic traffic model. From the simulation results, we calculate the three metrics: average evacuation time, evacuation completion time, and average time spent on the road by evacuees. Table \ref{tab:results_comparison_considering_delays} shows a comparison of \algosim{}-$10\%$ and \algo{} solutions in terms of these metrics. We can readily see that \algosim-$10\%$ outperforms \algo{} by $10\%, 17\%, 77\%$ in terms of the three metrics, respectively.

\begin{table}[!h]
    \centering
    \begin{tabular}{ccc}
    \toprule
    \algo{} & \algosim{}-$10\%$ &  Improvement $(\%)$\\
    \midrule
    \multicolumn{3}{l}{\textbf{Average evacuation time (hours)}}\\
    \midrule
    9.45 & 8.47 & 10.37\\
    \midrule
    \multicolumn{3}{l}{\textbf{Evacuation completion time (hours)}}\\
    \midrule
    24.29 & 20.15 & 17.04\\
    \midrule
    \multicolumn{3}{l}{\textbf{Average time spent on the road (hours)}}\\
    \midrule
    8.01 & 1.85 & 76.88\\
    \bottomrule
    \end{tabular}
    \caption{\algosim{}-$10\%$ outperforms \algo{} in terms of the average evacuation time, evacuation completion time, and average time spent on the road.}
    \label{tab:results_comparison_considering_delays}
\end{table}

\section{Proof of Hardness}
\label{app:hard}
In this section, we show that the problems we consider are not only {\sf NP-hard} but also hard to approximate. Even when we consider special planar graphs that perhaps is closer to a city's road network where $G$ is a subgraph of a grid and all destinations are along the border, these problems remain {\sf NP-hard}.
A summary of the hardness results is found in Table \ref{tbl:hard_appendix}. 

\begin{table}
\begin{center}
\begin{tabular}{|p{1.5cm}|p{1.35cm}|p{1.2cm}|p{1.35cm}|p{1.2cm}|}
    \hline
    Underlying & Hardness & \multicolumn{3}{c|}{Problems} \\ 
    \cline{3-5}
    Graph&  & \avgcostproblem & \completiontimeproblem & \partialavgcostproblem \\ 
    \hline
    Subgrid / Planar& {\sf NP-hard} &  Thm. \ref{thm:gridhard}& Thm. \ref{thm:gridhard} & Thm. \ref{thm:gridhard} \\ 
    \hline
    General with Two Sources / Sinks & {\sf $(3/2-\epsilon)$-hard} to approx. & Thm. \ref{hardconst} & See \cite{Golin2017NonapproximabilityAP}  & Thm. \ref{hardconst} \\ 
    \hline
    General & {\sf $O(\log n)$-hard} to approx. & Thm. \ref{hardlog} & See \cite{Golin2017NonapproximabilityAP}  & Thm. \ref{hardlog} \\ 
    \hline
\end{tabular}
\caption{Summary of Hardness}
\label{tbl:hard_appendix}
\end{center}
\end{table}

\begin{theorem}
\label{thm:gridhard}
Problems \avgcostproblem{}, \completiontimeproblem{} and \partialavgcostproblem{} are {\sf NP-hard} even If $G$ is a subgraph of a grid and all safety destinations are along the outer boundary. 
\end{theorem}


\begin{theorem}
\label{hardconst}
For any $\epsilon > 0$, it is {\sf NP-hard} to approximate \avgcostproblem{} and \partialavgcostproblem{} to a factor of ($3/2-\epsilon$) of the optimum, even when there are only two sources and one safe node. 
\end{theorem}


To prove Theorem \ref{thm:gridhard}, we rely on the general $k$-Node-Disjoint Path Problem:

\begin{problem} [$k$-Node-Disjoint Path Problem (kNDP)]
Given a graph $G$, a set of $k$ source-sink pairs $s_i, t_i$, find node-disjoint paths from each source $s_i$ to sink $t_i$. 
\end{problem}

The above problem was proven to be hard even when $G$ is a subgraph of a grid \cite{chuzhoy2017new}. 

\begin{theorem}
The $k$-Node-Disjoint Path Problem is hard to approximate to a factor of $2^{\Omega(\sqrt{\log n})}$ even when the graph is a subgraph of a grid and all sources lie on the outer boundary. 
\end{theorem}

We point out that \completiontimeproblem{} is equivalent to the Confluent Quickest Flow Problem in \cite{Golin2017NonapproximabilityAP}, which the authors have shown the above inapproximability result. In addition, they have also provided a bicriteria hardness result where it is not possible to approximate the completion time within some constant factor even when satisfying only some constant fraction of the demand. For \partialavgcostproblem{}, note that it is a generalization of Problem \avgcostproblem{} where $p = 1$. Therefore, any hardness result for \avgcostproblem{} also holds for \partialavgcostproblem. Thus, we focus on the following two theorems:

The approximation hardness proof of \avgcostproblem{} is similar to the one in \cite{Golin2017NonapproximabilityAP}. The main difference is in its analysis since the objective in the two problems differ. For brevity and completeness, we outline the reductions but omit certain proofs and ask the readers to refer to \cite{Golin2017NonapproximabilityAP} for more details.

The approximation hardness result relies on the NP-hardness of the capacitated version Two Distinct Path Problem. 

\begin{problem}[Two Distinct Path Problem]
Let $G$ be a graph with two sources $x_1, x_2$ and two sinks $y_1, y_2$. Every edge is labelled with either $1$ or $2$. Determine if there exists two edge-disjoint paths $P_1, P_2$ such that $P_i$ is a path from $x_i$ to $y_i$ for $i=1, 2$ and $P_2$ only uses edges with label $2$ ($P_1$ can use any edge). 
\end{problem}

The above problem is known to be {\sf NP-hard} \cite{guruswami2003near}. Other variations of the problem such as uncapacitated, undirected/directed, edge/node-disjoint paths are also known to be hard (see e.g. \cite{fortune1980directed}, \cite{naves2010maximum} and \cite{robertson1995graph}). 

\begin{proof}[Proof of Theorem \ref{hardconst}]
Given an instance $\mathcal{I}$ of the Two Disjoint Path Problem, consider constructing the following graph $G$ where we attach safe node $t$ to $y_1, y_2$ with an edge of capacity $1, 2$ respectively. For $i=1, 2$, we also add a source $s_i$
and attach it to $x_i$ with an edge of capacity $i$. Every edge with label $i$ also has capacity $i$. Source $s_i$ has $M*i$ evacuees for some large $M$ and each edge has a travel time of $1$. The upperbound completion time is set to be $M^2n$. 

If there exists two disjoint paths in $\mathcal{I}$, then a valid schedule simply sends $i$ evacuees at every time step, where the last group of people leaves their sources at time $M$. Since each path has length at most $n=|V(G)|$, the $i$ evacuee that left source $s_i$ at time $k$ is guaranteed to arrive by time $k+n$. Then, the total evacuation time is at most $\sum_{k=1}^M3(k+n) = 3M^2/2+3M/2 +3Mn$, resulting in an average evacuation time of roughly $M/2$. 

If $\mathcal{I}$ does not have two disjoint paths, then consider the two paths $P_1, P_2$ in a solution to \avgcostproblem{} in $G$. If $P_1, P_2$ intersects before $t$, since it is a confluent flow, the edge following their node of intersection is a single-edge cut that separates the sources from the sink. If the two paths only intersects at $t$, since we are in a NO-instance of $\mathcal{I}$, $P_2$ must have used an edge of capacity $1$. Then, deleting that edge along with $s_1x_1$ also separates the sources from the sink. Note that in both cases, the cut has capacity at most $2$. Then, at every time step, at most $2$ evacuees can cross the cut. Thus, for all $3M$ evacuees to cross the cut, it takes at least $3M/2$ time steps. Due to this bottleneck, it follows that the smallest total evacuation time is at least $\sum_{k=1}^{3M/2} 2(k+n) \ge 9M^2/4$, giving an average of at least $3M/4$. 

Since the two instances has a gap of $3/2-\epsilon$ where $\epsilon$ depends on the choice of $M$, our result follows. 
\end{proof}

The proof of Theorem \ref{hardlog} follows a similar structure. We use the same setup as the proof of  {\sf $\log$-hardness} 
of Quickest Flow Time in \cite{adiga13transportation} (Theorem 7) with an arbitrarily large upperbound on completion time. Refer to \cite{adiga13transportation} for more details. Note that the resulting graph contains $N$ sources, one safe node and a total of $M^2\log n$ evacuees. We can similarly show that a YES-instance of the Two-Disjoint Path problem leads to an average evacuation time of at most $M^2/2$. Meanwhile, if it is a NO-instance, then by Lemma 23 in \cite{adiga13transportation}, there is cut of capacity at most 2, creating a bottleneck. Using similar analysis as before, one can show that the average evacuation time is at least $M^2\log N/4$. Then, theorem~\ref{hardlog} follows immediately. 

\begin{proof}[Proof of Theorem~\ref{thm:gridhard}]
The proofs for all three problems are similar and hence we only focus on \avgcostproblem{} here. Given an instance of $k$NDP where $G$ is a subgraph of a grid and $k=O(n)$, we first swap the location of the sources and sinks such that all sinks lie on the outer boundary. One can easily check that subdividing any edge and adding a leaf to any vertex of degree less than $4$ still ensures that $G$ remains a subgraph of a grid. Then, we claim that without loss of generaltiy, we can assume that all sources and sinks are degree-$1$ vertices. This can be accomplished by first subdivide every edge and shift the sources/sinks to an edjacent newly added vertex. Then, add a single edge to it and shift the source/sink to the new pendant vertex. This ensures that every source and sink is incident to only one edge. For any edge incident to a source $s_i$ or a sink $t_i$, assign it a capacity of $i$; every other edge has a capacity of $k$. Each source $s_i$ has demand $M*i$ where $M \ge n^3$. This means in total, there are $M*k^2/2$ evacuees. 

In a YES-instance, every source follows its designated path. At every timestep, each source $s_i$ sends $i$ people, ensuring a total of $k^2/2$ people leaves the sources every timestep. Since each disjoint path has length at most $n$, anyone leaving at time $t$ is guaranteed to arrive by time $t+n$. This implies that the last group, leaving at time $M$ also arrives by time $M+n$, achieving an average arrival time of at most $\frac{1}{M*k^2/2} (k^2/2)*\sum_{t=1}^M (t+n) \le M/4 + n $. 

In a NO-instance, in a solution to \avgcostproblem, consider the cut formed by the edges incident to the sinks. We claim that the amount of flow through this cut at any point in time is at most $k^2/2 -1$. Assume for the sake of contradiction that there exists some point in time where the flow across the cut is at least $k^2/2$. Since the total capacity of the cut is $k^2/2$, then every edge is used at full capacity at that point in time; in particular, every sink is used in the final routing. Since the routes are confluent, every source is matched to a distinct route. Since we are in a NO-instance, there exists $i, j$ such that source $s_i$ is routed to sink $t_j$ and $i<j$. Since source $s_i$ can send at most $i$ flow at any point in time, sink $t_j$ can never receive its full capacity of $j$, a contradiction. 

Then, at any point in time, at most $k^2-1$ people can arrive at the destinations. This implies we need at least $T=M*k^2/(2(k^2-1))$ rounds and thus a total evacuation time of at least $(T^2/2)(k^2-1)$. This implies an average evacuation time of at least $Mk^2/(4(k^2-1)) = M/4 + M/(4(k^2-1))$. By our choice of $M$, $M/(4(k^2-1)) \ge n$, causing a gap between the YES and NO instances, proving our theorem. 
\end{proof}

\end{document}